%% file: main.tex
\documentclass[10pt]{article}

\usepackage[a4paper,left=1.9cm,right=1.9cm,top=2.0cm,bottom=2.2cm]{geometry}

\usepackage{mathtools}          
\usepackage{amsthm}

\usepackage{iftex}
\ifPDFTeX
  \usepackage[T1]{fontenc}
  \usepackage[utf8]{inputenc}
  \usepackage{newtxtext,newtxmath}
  \usepackage[scaled=0.92]{helvet}   
\else
  \usepackage{fontspec}
  \usepackage{newtxtext,newtxmath}
\fi
\newcommand{\sans}{\sffamily}
\usepackage{bm}

\usepackage{graphicx}
\usepackage{xcolor}
\usepackage[most]{tcolorbox}
\usepackage{tikz}
\usetikzlibrary{positioning,arrows.meta,calc,fit,backgrounds}
\usepackage{pgfplots}
\pgfplotsset{compat=1.17}
\usepackage{booktabs}
\usepackage{array}
\usepackage{ragged2e}
\usepackage{enumitem}
\usepackage{microtype}
\usepackage{colortbl}
\usepackage[font=small,labelfont=bf]{caption}

\definecolor{scnavy}{HTML}{155050}
\definecolor{scblue}{HTML}{004747}
\definecolor{sclink}{HTML}{00A3A3}
\definecolor{scbox}{HTML}{EEF1F5}     
\definecolor{scrule}{HTML}{CDD9E5}
\definecolor{scgray}{HTML}{8A97A6}
\definecolor{scrowalt}{HTML}{F4F7FA}

\usepackage[colorlinks=true,linkcolor=scnavy,urlcolor=sclink,citecolor=sclink]{hyperref}

\usepackage{algorithm}
\usepackage{algpseudocode}

\usepackage{titlesec}
\titleformat{\section}{\sans\large\bfseries\color{scnavy}}{\thesection}{0.6em}{}
\titleformat{\subsection}{\sans\normalsize\bfseries\color{scblue}}{\thesubsection}{0.6em}{}
\titleformat{\subsubsection}{\sans\small\bfseries\color{scblue}}{\thesubsubsection}{0.6em}{}
\titlespacing*{\section}{0pt}{14pt}{5pt}
\titlespacing*{\subsection}{0pt}{9pt}{3pt}

\theoremstyle{plain}
\newtheorem{proposition}{Proposition}[section]

\theoremstyle{definition}

\newtheorem{assumption}[proposition]{Assumption}
\theoremstyle{remark}
\newtheorem{remark}[proposition]{Remark}

\newcolumntype{L}[1]{>{\raggedright\arraybackslash}p{#1}}
\newcolumntype{R}[1]{>{\raggedleft\arraybackslash}p{#1}}

\newcommand{\thd}[1]{\textbf{\textcolor{white}{#1}}}      
\newcommand{\hdrow}{\rowcolor{scnavy}}

\newcommand{\E}{\ensuremath{\mathbb{E}}}
\newcommand{\Var}{\ensuremath{\operatorname{Var}}}
\newcommand{\Cs}{\ensuremath{C^{\star}}}                   
\newcommand{\Chat}{\ensuremath{\widehat{C}_k}}             
\newcommand{\mk}{\ensuremath{m_k}}                         
\newcommand{\Uk}{\ensuremath{U_k}}                         
\newcommand{\kstar}{\ensuremath{k^{\star}}}                
\newcommand{\eps}{\ensuremath{\varepsilon}}                
\newcommand{\Nhat}{\ensuremath{\widehat{N}}}               

\newtcolorbox{authornote}{colback=scrowalt,colframe=scrule,boxrule=0.6pt,arc=3pt,
  left=8pt,right=8pt,top=6pt,bottom=6pt,fontupper=\small}

\usepackage{fancyhdr}
\begin{document}

\begin{tcolorbox}[
    enhanced, colback=scbox, colframe=scbox, boxrule=0pt,
    arc=7pt, left=16pt, right=16pt, top=15pt, bottom=13pt,
    drop shadow={black!12}]

{\sans\bfseries\fontsize{20}{24}\selectfont \color{scnavy}
The German Tank Problem \& Navigating by Scent\par}
\vspace{5pt}
{\sans\large\color{scblue}How an old defense strategy informed a critical problem in robotic navigation\par}
\vspace{6pt}

Kordel K.\ France\\
\textit{Scentience, Inc.}
\vspace{8pt}

\hrule height 0.4pt \vspace{8pt}
\input{sections/main/00_abstract}

\vspace{9pt}

\hrule height 0.4pt \vspace{7pt}
{\setlength{\parskip}{2pt}
\textbf{\sans Date:}\; 29 October 2026 \\
\textbf{\sans Version:}\; v2.1 \\
\textbf{\sans Companion paper:}\; \href{https://www.mdpi.com/1424-8220/26/9/2849}{Chasing Ghosts} \\
\textbf{\sans Correspondence:}\; \href{mailto:kordel@scentience.ai}{kordel@scentience.ai} \\
\textbf{\sans Access:}\; \href{https://scentience.ai/api-docs}{Scentience API} \\
\par}

\IfFileExists{figures/scentience_logo_white_transparent_padding.png}{%
\vspace{-50pt}
\hfill\raisebox{0pt}{\includegraphics[height=40pt]{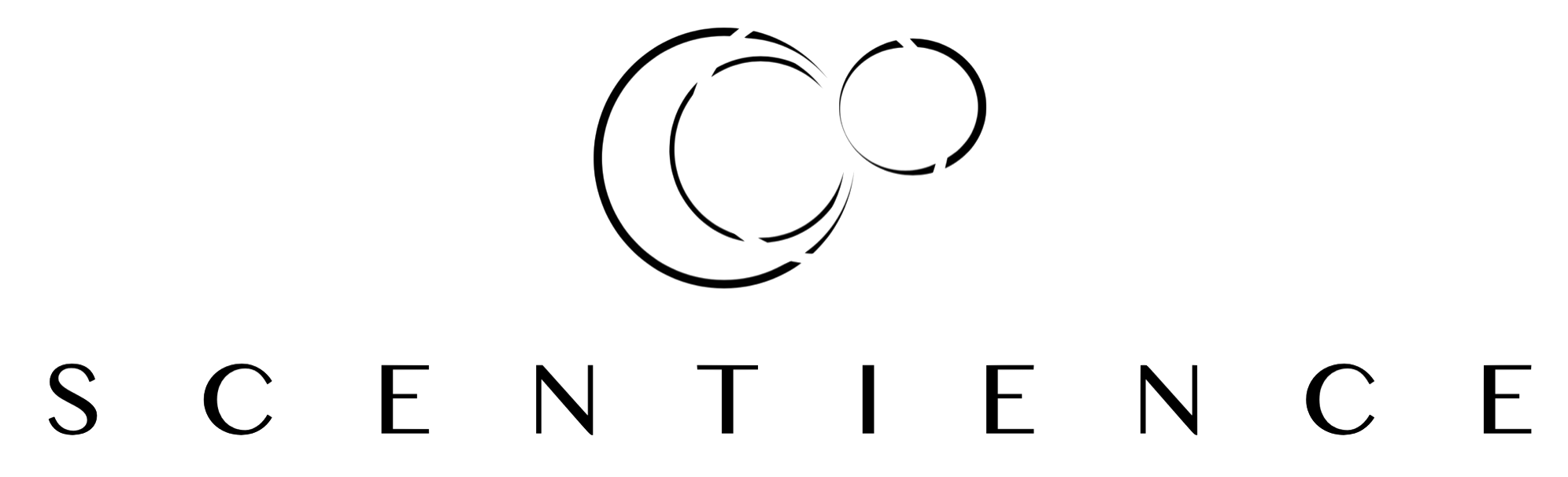}}
\vspace{5pt}}{}

\end{tcolorbox}

\input{sections/main/01_intro}
\input{sections/main/02_background}
\input{sections/main/03_method}
\input{sections/main/04_decision}
\input{sections/main/06_limitations}

\input{sections/main/07_conclusion}
\input{sections/main/08_llm_declaration}

\label{main-text-end}
\vspace{-4mm}
{\small
\bibliographystyle{plain}
\bibliography{main}
}

\clearpage
\appendix
\label{supp-start}
\setcounter{figure}{0}\setcounter{table}{0}\setcounter{equation}{0}\setcounter{algorithm}{0}
\renewcommand{\thefigure}{S\arabic{figure}}\renewcommand{\thetable}{S\arabic{table}}
\renewcommand{\theequation}{S\arabic{equation}}\renewcommand{\thealgorithm}{S\arabic{algorithm}}
\renewcommand{\theHfigure}{S\arabic{figure}}\renewcommand{\theHtable}{S\arabic{table}}
\renewcommand{\theHequation}{S\arabic{equation}}\renewcommand{\theHalgorithm}{S\arabic{algorithm}}

\begin{tcolorbox}[
    enhanced, colback=scbox, colframe=scbox, boxrule=0pt,
    arc=7pt, left=16pt, right=16pt, top=15pt, bottom=13pt,
    drop shadow={black!12}]
{\sans\bfseries\fontsize{17}{21}\selectfont \color{scnavy} Supplementary Material \par}
\vspace{4pt}
\vspace{7pt}
\hrule height 0.4pt \vspace{7pt}
\small
The artifacts of the supplementary material are organized  in the order by which they are referred within the main text.
\end{tcolorbox}
\suppressfloats[t]

\vspace{6pt}

\input{sections/supplementary/supplementary}

\end{document}

%% file: sections/main/00_abstract.tex

\textbf{Abstract.}
A robot searching for the source of an odor must eventually answer a notably hard question: \textit{have I arrived?}
Near the source, concentration gradients homogenize, turbulence corrupts the signal, and a small robot often lacks the positional accuracy to map the plume.
Here, we describe a lightweight remedy inspired by a World War II statistics conundrum: the German tank problem.
To approximate German armament production, Allied analysts developed an estimator based on serial numbers from captured tanks.
Analogously, we use the same procedure over sampled chemical sensor readings from a robot to create estimates for the peak concentration of an odor plume, and therefore the odor's source.
Treating each sensor reading as an observation, our unbiased estimator needs only a running maximum and a sample counter.
A confidence bound turns both numbers into a principled rule that the robot or UAV can use to determine whether it should keep searching for the odor source.
We derive the estimator and its bound, demonstrate its sample efficiency, and examine its robustness even when its assumptions fail.
Finally, we summarize the estimator's role as the termination criterion for olfactory navigation.


%% file: sections/main/01_intro.tex
\section{Introduction}\label{sec:intro}

\begin{figure}[b]
\centering
\includegraphics[width=1.0\linewidth]{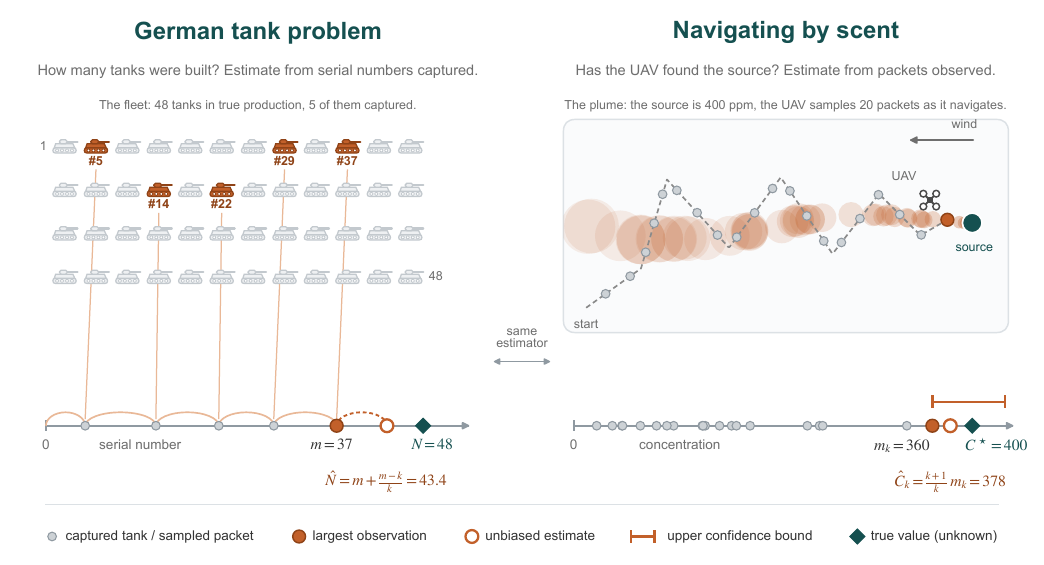}
\caption{\small The German tank problem uses the same estimator as a robot determining when it has reached the source of an odor plume.}
\label{fig:gtp_concept}
\end{figure}

In any navigation task, a robot must have termination criteria for knowing it has arrived at its destination.
One of the most difficult problems for a robot navigating by scent is having it estimate when it is ``done''.
In other words, how does the robot know when it has reached the global source of the odor, and not just a local maximum concentration of the plume?
In the most simplistic perspective, this is analogous to one trying to climb a mountain in the fog without the ability to ever see the top - how would one know if they have reached it?
If a robot is indeed near the source but does not yet know, it may circle around trying to maximize the gradient, in turn distorting the plume and creating more turbulence in the process.

Declaring ``done'' is hard for compounding reasons.
A typical undisturbed plume becomes simultaneously more concentrated and more homogeneous toward its source.
Before we formalized the algorithm in this paper, our olfaction-only UAV would spend several seconds circling the source, chasing ever-smaller gradients~\cite{france2026chasing}.
Many times, this distorted the plume so much that it created currents that pulled the UAV away from the source it was already near, and drained more of the battery.
Coupled with the known drift of chemical sensors ~\cite{dennler2022drift}, this mandates a sense of urgency for the robot to find the odor source as fast and with as little movement as possible.

Turbulent plumes arrive as whiffs separated by blanks~\cite{celani2014odor,crimaldi2022active, schmuker2016exploiting}, so no single reading can be decisive by itself.
In addition, small platforms often cannot precisely indicate \emph{where} a sample was taken: on our UAV, payload-induced bias in the inertial measurement unit made long-horizon position estimates unusable~\cite{france2026chasing,france2025oiocalib}, ruling out methods that fuse concentrations with sample locations.
Concentration-based maps hold little utility in environments with wind or HVAC systems due to the constant change in chemistry, and a robot may also not know what concentration to expect at the source.
Termination criteria are therefore warranted to help the robot know whether it has arrived at the plume source or should keep searching.

There is a classic estimation procedure developed in World War II that models this well: the German tank problem.


%% file: sections/main/02_background.tex
\section{The German Tank Problem}\label{sec:tanks}

\begin{table}[b]
\centering\small
\caption{The German tank problem and olfactory navigation share a common thesis: estimate the upper end of a range from the largest value sampled so far.}
\label{tab:analogy}
\begin{tabular}{L{3.2cm}L{5.4cm}L{6.6cm}}
\hdrow \thd{Variable} & \thd{German tank problem} & \thd{Odor source declaration}\\
\textbf{Population} & tanks produced & chemical sensor readings\\
\rowcolor{scrowalt} \textbf{Label on each unit} & serial number & chemical concentration\\
\textbf{Unknown upper end} & total number of produced tanks, $N$ & max concentration at the plume source, $\Cs$\\
\rowcolor{scrowalt} \textbf{Observations} & $k$ captured tanks & $k$ sensor readings\\
\textbf{Sampling} & without replacement & with replacement (readings of a given strength recur)\\
\rowcolor{scrowalt} \textbf{Sufficient summary} & $k$ and the largest serial number, $m$ & $k$ and the running max concentration, $\mk$\\
\end{tabular}
\end{table}

Beginning in 1943, the Economic Warfare Division of the U.S.\ Embassy in London studied the serial numbers on captured German equipment, from tires to tanks, trucks, guns, and rockets, to estimate total enemy production~\cite{ruggles1947empirical}.
For tanks, the most useful numbers were on gearboxes, which typically ran in continuous sequences.
Conventional intelligence consistently produced inaccurate predictions.
Statisticians at the embassy developed an estimation procedure to help produce more reliable metrics as a result.
For June 1940, this estimator put monthly tank production at 169, against a conventional intelligence estimate of 1{,}000; German records examined after the war showed there were truly 122 tanks.
For June 1941, these respective figures were 244, 1{,}550, and 271, and for August 1942, there were 327, 1{,}550, and 342~\cite[p.~89]{ruggles1947empirical}.
The estimator proved so successful that it was later carried forward by other intelligence groups among more scenarios.
Leo A. Goodman formalized the estimator for the public within his 1952 paper ``Serial Number Analysis'' ~\cite{goodman1952serial}.\footnote{Modern frequentist and Bayesian translations can be found  in~\cite{clark2021lessons,simon2024bayesian}.}

The estimator follows from a simple argument.
Suppose tanks are numbered $1,\dots,N$ and $k$ distinct tanks are captured at random, with $m$ representing the largest serial number among them.
We can divide the serial numbers $[1,m]$ by $k$ to create $k$ gaps (counting the stretch below the smallest), which together hold the $m-k$ unseen serial numbers, or $(m-k)/k$ per gap on average.
Nothing distinguishes the unseen stretch above $m$ from the others, so we add one additional gap for the average:
\begin{equation}
  \Nhat \;=\; m + \frac{m-k}{k} \;=\; m\Bigl(1+\frac{1}{k}\Bigr) - 1 .
  \label{eq:tank}
\end{equation}
Formally, the probability that the observed serial number $m$ is indeed the largest serial number $M$ can be calculated by $\Pr(M=m)=\binom{m-1}{k-1}\big/\binom{N}{k}$ for $k\le m\le N$.
We can then calculate the expected value, $\E[M]=k(N+1)/(k+1)$, where $\Nhat$ is unbiased.
$M$ is a complete and sufficient statistic, which effectively means it is the best estimator that captures all knowledge needed to estimate the desired value.
Due to this, $\Nhat$ can be considered as the minimum-variance unbiased estimator~\cite{goodman1952serial,lehmann1998theory}, with $\Var(\Nhat)=(N-k)(N+1)/\bigl(k(k+2)\bigr)\approx N^2/k^2$.

In consequence, there is a characteristic of Equation ~\eqref{eq:tank} that really matters for robots navigating by scent.
Everything the sample says about the estimated number of tanks $N$ can be characterized by two numbers: the maximum observed serial number $m$ and the number of captured tanks $k$.
As we will show in the next section, the requirement to maintain only two numbers to estimate ``done'' creates a very computationally efficient algorithm that can run on virtually any platform - a highly desirable feature for robotics.

%% file: sections/main/03_method.tex

\section{From Tanks to Odor Packets}\label{sec:packets}

Odor source localization is commonly split into three sub-problems: detecting a plume, staying in contact with it while moving upwind, and \textit{declaring} that the source has been reached~\cite{hayes2002distributed,kowadlo2008robot} (we call this, the ``termination criteria'').
The first two have drawn most of the attention within literature, from bio-inspired surge-and-cast controllers and bout detection~\cite{schmuker2016exploiting,shigaki2018design} to information-seeking search~\cite{vergassola2007infotaxis}, probabilistic source-term estimation~\cite{hutchinson2017review}, and learned policies~\cite{singh2023emergent}, with aerial implementations in~\cite{burgues2019smelling,duisterhof2021sniffy} (see~\cite{francis2022gas} for a recent thorough review).
The termination criteria receive far less research attention, yet it decides whether a search ends in success or in a robot spending additional time and energy circling a few meters from its goal.
Surprisingly, this outcome has a large reciprocal influence on the complexity of the algorithms required to perceive and navigate leading up to the termination criteria.
This can result in either success or failure for the robot, which makes the problem worth solving.

An odor plume is a stream of filaments, or packets, carrying a compound (or series of compounds) downwind from its emitter.
At each sampling instant, a chemical sensor reports the concentration of the air packet currently passing over it.
We treat each reading as one observed packet, and its concentration as that packet's ``serial number'' analog (see Table~\ref{tab:analogy}).
This idea is an extension of the paper from \cite{france2026chasing} and we reference their demonstration of the success of this estimator throughout.
Our analogy rests on a few assumptions, the first being about the physical environment (see section \ref{sec:limits} for limitations on these assumptions).

\begin{assumption}\label{ass:physical}
   A passive, non-reacting compound obeys an advection-diffusion equation, $\partial_t c + \mathbf{u}\cdot\nabla c = D\,\nabla^2 c$, with concentration $c$ and a molecular diffusivity $D$, assuming no source terminal apart from the emitter; under such an equation, the largest value of $c$ will be observed at the emitter, or in minimally concentrated air at the initial condition of navigation~\cite{evans2010pde}. 
\end{assumption}

Every packet the robot samples therefore carries at most $\Cs$, the concentration at the source, while packets sampled near the source come arbitrarily close to it, just as no serial number exceeds $N$ while late-production tanks come close.
Because a chemical sensor responds monotonically to concentration over its working range, the same holds in sensor units.

\begin{assumption}\label{ass:model}
After subtracting the clean-air baseline, the readings $c_1,\dots,c_k$ are
(A1)~bounded above by the unknown source value $\Cs$;
(A2)~independent and identically distributed, so a packet of a given strength may be seen again, as in sampling with replacement~\cite{france2026chasing}; and
(A3)~uniformly distributed on $[0,\Cs]$.
\end{assumption}

A1 denotes how the readings are bounded given the environment defined by A0.
A2 reflects that, unlike a captured tank, an odor packet of a given strength can recur.
A3 establishes that no concentration level lies below that at the plume source.
Section~\ref{sec:limits} shows what happens when Assumptions \ref{ass:physical} and \ref{ass:model} fail.
The model is built to be scale- and unit-free such that a gain on every reading scales the estimate and its bounds identically, enabling the method to work among any chemical sensing modality at any measurement frequency without the need for calibration as shown in ~\cite{france2026chasing}.

\section{The Estimator}\label{sec:estimator}

See the appendix for formal proofs of key equations.

\subsection{Two numbers are enough}
Under Assumption~\ref{ass:model}, the likelihood of the maximum concentration $\Cs$ given the observed sensor readings is
\begin{equation}
  L(\Cs) \;=\; \prod_{i=1}^{k} \frac{1}{\Cs}\,\mathbf{1}\{0\le c_i\le \Cs\}
  \;=\; (\Cs)^{-k}\,\mathbf{1}\{\mk \le \Cs\},
  \qquad \mk = \max_{i\le k} c_i .
  \label{eq:lik}
\end{equation}
The sensor readings enter only through $\mk$ and $k$, so by the factorization theorem $\mk$ is sufficient for $\Cs$~\cite{lehmann1998theory}.
Effectively, a robot that keeps only a counter and a running maximum has discarded nothing the readings could tell it about the source.
The likelihood peaks at $\Cs=\mk$, but this maximum-likelihood estimate is biased low, since the strongest packet seen can never exceed the source and is almost never equivalent to it.

\subsection{Sampling distribution and the unbiased estimator}
Because the sensor readings are independent, we can compute their probability such that
\begin{equation}
  \Pr(\mk \le x) \;=\; \prod_{i=1}^{k}\Pr(c_i\le x) \;=\; \Bigl(\frac{x}{\Cs}\Bigr)^{k},
  \qquad 0\le x\le \Cs,
  \label{eq:cdf}
\end{equation}
so $\mk/\Cs$ has density $k\,u^{k-1}$ on the range $[0,1]$ and $\E[\mk]=k\Cs/(k+1)$.
On average, the running maximum falls short of the source by $\Cs/(k+1)$, which is the continuous counterpart of the missing gap $\Nhat$ from Equation \ref{eq:tank} in Section~\ref{sec:tanks}.

\begin{proposition}\label{prop:umvue}
Under Assumption~\ref{ass:model},
\begin{equation}
  \Chat \;=\; \frac{k+1}{k}\,\mk
  \label{eq:umvue}
\end{equation}
is the unique minimum-variance unbiased estimator of $\Cs$, and $\Var(\Chat)=(\Cs)^2/\bigl(k(k+2)\bigr)$.
\end{proposition}

Equation~\eqref{eq:umvue} is the continuous form of Equation~\eqref{eq:tank}. 
Under a requirement to estimate the true highest concentration with at least $k=20$ measurements, this leaves us with $\widehat{C}=1.05m$, meaning that our estimator is likely to underestimate the true highest concentration of the odor source $m$ by $\approx 5\%$ as we show in \cite{france2026chasing}. 

\subsection{An exact confidence bound}
Point estimates are fragile when the number of sensor readings $k$ is small, and a robot deciding whether to land needs to know how far above $\mk$ the true odor concentration could plausibly lie.
Equation~\eqref{eq:cdf} addresses this because the pivot $(\mk/\Cs)^k$ is uniform on $[0,1]$ $\forall c \in \Cs$.

\begin{proposition}\label{prop:ci}
Under Assumption~\ref{ass:model}, for any $0<p<q\le 1$,
\begin{equation}
  \Pr\!\left( \frac{\mk}{q^{1/k}} \;\le\; \Cs \;\le\; \frac{\mk}{p^{1/k}} \right) \;=\; q-p .
  \label{eq:ci}
\end{equation}
In particular, at confidence $1-\alpha$, the one-sided interval $[\mk,\,\Uk]$ with
\begin{equation}
  \Uk \;=\; \mk\,\alpha^{-1/k}
  \label{eq:upper}
\end{equation}
contains $\Cs$ with probability exactly $1-\alpha$, and is the shortest interval at that confidence.
\end{proposition}

In~\cite{france2026chasing} we used the two-sided 95\% interval ($p=0.025$, $q=0.975$).
With $k=20$ readings, $\Chat=1.05\,\mk$ and Equation~\eqref{eq:ci} gives $[\,1.001\,\mk,\;1.203\,\mk\,]$: with 95\% confidence, the source is at most about 20\% stronger than the strongest packet seen.
For declaration only the upper end matters, and the upper end of a two-sided 95\% interval is a one-sided 97.5\% bound.
Figure~\ref{fig:bound} of the appendix shows a graphical representation of these bounds as they relate to the demonstration in \cite{france2026chasing}.

%% file: sections/main/04_decision.tex
\section{From Estimate to Decision}\label{sec:decision}

It is prudent to show how this estimate is used by the robot to actually make a decision.
Let $\bar c_k$ be the robot's current baseline-corrected reading.
The bound established by Equation~\eqref{eq:upper} effectiely becomes a gate on the navigation policy's decision to stop.
When the policy proposes to land (for a gradient follower, when it finds no direction of increasing concentration), the robot lands only if
\begin{equation}
  \bar c_k \;\ge\; \frac{\Uk}{1+\eps} \;=\; \frac{\mk\,\alpha^{-1/k}}{1+\eps},
  \label{eq:rule}
\end{equation}
that is, only if the air it is in now is within a factor $1+\eps$ of the strongest air the plume plausibly contains (which allegedly lies at the plume's source).
Otherwise the robot keeps searching.
In~\cite{france2026chasing} we used $\alpha=0.025$ and $\eps=0.25$, which establishes that the robot must be within 25\% of the plume's plausible maximum.
In the worked example of Section~\ref{sec:estimator}, $U_{20}=1.203\,\mk\le 1.25\,\mk$, so a UAV whose current reading was its running maximum would land.

\paragraph{Minimum evidence required.}
A robot may falsely apply this rule at the beginning of navigation if it does not have enough time to properly explore.
Consequently, we establish a minimum value of $k$ that is used by the robot to gather initial evidence.
Because our concentration estimate $\bar c_k$ must be less than the true maximum concentration $m_k$after $k$ samples, condition~\eqref{eq:rule} above can hold only once $\alpha^{-1/k}\le 1+\eps$.
We can succinctly re-arrange this relationship as
\begin{equation}
  k \;\ge\; \kstar \;=\; \left\lceil \frac{\ln(1/\alpha)}{\ln(1+\eps)} \right\rceil .
  \label{eq:kstar}
\end{equation}
For $\alpha=0.025$ and $\eps=0.25$, $\kstar=\lceil 16.53\rceil=17$.
The rule thus does two jobs.
Firstly, $k\ge\kstar$ is a minimum-evidence requirement fixed before deployment, independent of the data such that no decisions can be made by the robot on fewer than 17 measurements.
This stops the robot from landing prematurely on a local maximum eddy current.
Secondly, the rule asks the robot to be near the strongest concentration it has encountered at the 17th measurement; otherwise, it must keep exploring.
One can observe the effects of tuning $\eps$ and $\kstar$ as hyperparameters: $\eps=0.10$ raises $\kstar$ to 39, and $\alpha=0.05$ lowers it to 14.

\paragraph{Computational cost.}
Algorithm~\ref{alg:declare} within the appendix lists the complete procedure.
Each reading costs one comparison, one increment, and one multiplication by $\alpha^{-1/k}$, taken from a precomputed table or evaluated as $\exp\bigl(\ln(1/\alpha)/k\bigr)$.
This enables the entire state to be determined by only two numbers.
The rule needs no dispersion model, deep neural network, wind estimate, map, or position fix, which is what made it usable on a UAV whose odometry could not be trusted beyond the short term~\cite{france2026chasing,france2025oiomethod}.
By contrast, particle-filter source-term estimation propagates a dispersion model for every particle at every update~\cite{hutchinson2017review}, and infotaxis evaluates the expected entropy of a posterior over source locations for every robot move~\cite{vergassola2007infotaxis}.
Those methods arguably answer a different question (e.g. \textit{where} the source is) and our estimator does not attempt to replace them.
It simply seeks to answer whether the robot should stop exploring, and does so at a cost that fits on any microcontroller.
We note that the estimator can be run once per sensor reading, or on a cadence over a moving average of sensor readings.

%% file: sections/main/06_limitations.tex
\vspace{-3mm}
\section{Limitations}\label{sec:limits}

The estimator is only as good as Assumption~\ref{ass:physical} and ~\ref{ass:model}.
There are a few failure modes in which our estimator may lack performance, and they are worth considering for anyone trying to implement the same termination criteria for their robot.

\vspace{-4mm}
\paragraph{Simplified plume model (A0, A1).}
Our estimator assumes certain properties about molecular diffusion and filament transport within a plume.
For extremely dynamic environments with prolific compound mixing, wind changes, or heterogeneous temperatures/pressures/humidities, we suspect our estimator to exhibit poorer performance.
A truly accurate plume model that the robot could use to help reason at the edge would be desired, but as seen in most olfactory robotics work, this is not the common practice due to the computational constraints.

\vspace{-4mm}
\paragraph{Non-stationarity (A2) and non-uniformity (A3) of packet strengths.}
Far from the plume source most readings are weak or blank, while close to it they cluster nearer the source value~\cite{celani2014odor,crimaldi2022active}.
This creates non-uniformity throughout the samples used by our estimator.
We do not debate this, but find that our estimator held well in our experiments in \cite{france2026chasing} even while violating this principle.
During an approach to the plume source, the readings trend upward and the plume meanders, so they are not identically distributed over a whole flight.
Accumulating $k$ and $\mk$ over a sliding window, or resetting them when the bout detector reports plume loss, can help counter this.

\vspace{-4mm}
    \paragraph{Correlated readings (A2).}
A slow sensor can report the same filament over several samples, which is analogous to sampling many tanks with replacement.
Consequently, $k$ is likely to overstate the number of independent packets which may result in the confidence bound coming out too tight.
We found that spacing readings beyond the sensor's response time, as in the 1--2\,s cadence policy of~\cite{france2026chasing}, or counting bouts rather than raw samples as defined by ~\cite{schmuker2016exploiting}, keeps $k$ more grounded.
Section \ref{supp:estimates} of the supplementary material elaborates on how we distinguish sensor saturation from correlated readings.




\vspace{-4mm}
\paragraph{Vision helps (A2, A3).}
This method was developed specifically for a robot navigating by scent, using only two stereo chemical sensors, a single depth sensor, and two leveling sensors \cite{france2026chasing}.
Our intent was to show that a minimal sensing stack could enable successful olfactory localization.
However, as we also showed, vision helps to disambiguate at the last mile.
When leveraging a camera equipped with a sufficiently trained computer vision model for spotting plausible odor source emitters, the termination criteria become an augmentation to the final arrival decision versus the explicit rule alone.
We found that vision helps reduce the minimum value of sensor readings $k$ required for the estimator to be useful, and consequently shaved several seconds off each navigation task in compariosn to the olfaction-only experiments.

%% file: sections/main/07_conclusion.tex
\vspace{-3mm}
\section{Conclusion}\label{sec:conclusion}

A robot navigating by scent has to know when to stop, and it rarely has a map, a trustworthy position estimate, or prior knowledge of the source strength to tell it.
The German tank problem offers a principled answer template built from only two numbers: how many observations have been made, and the strongest observation among them.
The resultant is a minimum-variance unbiased estimator with an exact confidence bound and costs constant time and memory.
While scent-based navigation is the primary interest with which we compose our solution, our estimator can be used in any robotic navigation task where dynamic environments and sparse sensing come with the territory.
The methodology established here and its assumptions were proven over several stereo sensing experiments in \cite{france2026chasing}, where \textit{every} navigation trial resulted in the robot finding the source of the odor (see Section \ref{supp:examples} of the appendix for a summary of the results).
We acknowledge the estimator's surprising robustness given its simplicity, even when some of its assumptions are violated.
Natural next steps are integrating more accurate plume models with per-sensor estimators, and enabling additional parameters that can help counter hysteresis and aleatoric uncertainty. 
We are thrilled that a nearly century-old intelligence problem solved a key problem in our work, and hope that researchers working in olfactory robotics also find it a useful default for the last meter of a search.

Full code, models, and replication of the above method can be found in the repository associated with \cite{france2026chasing}.
\vspace{-3mm}
\paragraph{Acknowledgments.}
This method originated as part of the author's prior work in \cite{france2026chasing}, and has since been improved on and integrated into the methods and models of Scentience, Inc.
The author thanks its co-authors of \cite{france2026chasing} - Ovidiu Daescu, Latifur Khan, and Rohith Peddi - for contribution to and peer review of the original method.

%% file: sections/main/08_llm_declaration.tex
\section*{LLM Declaration}
Large language models (LLMs) were used to proofread the original text as well as to refine the proofs and plots, particularly in appendix sections \ref{supp:assumptions} \& \ref{supp:estimates}, respectively.
LLMs were also used to ensure notation consistency between our equations and those published by their original authors (e.g. Goodman \cite{goodman1952serial}, Ruggles \cite{Ruggles1947AnEA}, etc.), as the notation in equations from \cite{france2026chasing, france2026chasing_arxiv} may not have followed the standards set forth by their original authors.
Outside of these points, the idea, olfactory implementation, main text, algorithm, and mathematics of this manuscript were developed by the author unless otherwise referenced.

%% file: sections/supplementary/supplementary.tex
\section{Assumption Decomposition}\label{supp:assumptions}
\subsection{Proof of Proposition \ref{prop:umvue}}
\begin{proof}
Unbiasedness follows from $\E[\mk]=k\Cs/(k+1)$.
The family~\eqref{eq:cdf} is complete, so $\mk$ is complete and sufficient, and by the Lehmann--Scheff\'e theorem an unbiased function of it is the unique minimum-variance unbiased estimator~\cite{lehmann1998theory}.
For the variance, $\E[\mk^2]=k(\Cs)^2/(k+2)$ gives $\Var(\mk)=k(\Cs)^2/\bigl((k+1)^2(k+2)\bigr)$; multiply by $\bigl((k+1)/k\bigr)^2$.
\end{proof}

Equation~\eqref{eq:umvue} is the continuous form of Eq.~\eqref{eq:tank}: the $-1$ there accounts for the discreteness of serial numbers, amounts to one quantization step for integer ADC counts, and vanishes for continuous concentrations.\footnote{The termination criterion as printed in~\cite{france2026chasing} typesets the factor as $k/(k+1)$; the worked example there ($\widehat{C}=1.05\,m$ at $k=20$) uses the intended factor $(k+1)/k$.}
The variance explains the method's sample efficiency.
The relative standard deviation of $\Chat$ is $1/\sqrt{k(k+2)}\approx 1/k$.
The obvious alternative, twice the sample mean, is also unbiased but has relative standard deviation $1/\sqrt{3k}$.
At $k=20$ the two quotients are 4.8\% and 12.9\%, and the mean would need roughly 150 readings to match.
Effectively, $\Cs$ is the \textit{edge} of the distribution, and the extreme sample pins an edge far more tightly than an average pins a center.

\subsection{Proof of Proposition \ref{prop:ci}}
\begin{proof}
By Eq.~\eqref{eq:cdf}, $\Pr\bigl((\mk/\Cs)^k\le u\bigr)=u$ for $u\in[0,1]$, and the event $p\le(\mk/\Cs)^k\le q$ is the event in Eq.~\eqref{eq:ci}.
Setting $q=1$, $p=\alpha$ gives Eq.~\eqref{eq:upper}.
With $q-p=1-\alpha$ fixed, the length $\mk\bigl(p^{-1/k}-(p+1-\alpha)^{-1/k}\bigr)$ decreases in $p$ because $u\mapsto u^{-1/k}$ is convex, so it is smallest at $p=\alpha$.
\end{proof}

\section{Algorithm}\label{supp:algo}

The implementation of the estimator from the German tank problem contextualized against the stereo sensing olfactory mechanism of the UAV from \cite{france2026chasing} is shown in Algorithm \ref{alg:declare} below.
The algorithm is integrated into the Scentience APIs for rapid utility among users.

\begin{algorithm}[H]
\caption{Source declaration with the German-tank bound (runs once per sensor reading)}
\label{alg:declare}
\begin{algorithmic}[1]
\Require confidence level $1-\alpha$, tolerance $\eps$, clean-air baseline $b$
\State $k \gets 0$;\; $m \gets 0$ \Comment{the entire state}
\While{searching}
  \State $c \gets \textsc{Filter}(\text{reading}) - b$ \Comment{one reading $=$ one packet}
  \State $k \gets k+1$;\; $m \gets \max(m, c)$
  \State $\widehat{C} \gets m\,(k+1)/k$;\; $U \gets m\,\alpha^{-1/k}$ \Comment{Eqs.~\eqref{eq:umvue} and~\eqref{eq:upper}}
  \If{policy proposes \textsc{Land} \textbf{and} $c \ge U/(1+\eps)$}
    \State land; report $\widehat{C}$ and $[m, U]$ \Comment{source declared}
  \Else
    \State continue searching: surge, cast, or pause
  \EndIf
\EndWhile
\end{algorithmic}
\end{algorithm}

\section{Estimate Bounds}\label{supp:estimates}

The rule resulting from our bounded estimator demonstrates the strongest odor concentration this search has encountered, but not necessarily the strongest one that exists.
A robot trapped in a recirculation zone with a flat local maximum will satisfy Equation ~\eqref{eq:rule}, a weakness shared by every gradient-following method; exploration, or a second modality such as vision, can help address this.
In a similar manner, a sensor that saturates below the source yields $\mk$ at its upper bound, so the rule is satisfied once the sensor saturates.
This is usually acceptable, since saturation is itself strong evidence of proximity.
We add Remark~\ref{rem:bayes} to refine this rule this more explicitly.

\begin{remark}\label{rem:bayes}
With the scale-invariant prior $\pi(\Cs)\propto 1/\Cs$, the likelihood~\eqref{eq:lik} yields a posterior proportional to $(\Cs)^{-(k+1)}$ on $[\mk,\infty)$, so $\Pr(\Cs>x\mid c_{1:k})=(\mk/x)^k$ and the $1-\alpha$ posterior upper bound is exactly $\Uk$.
The Bayesian form also makes it easy to fold in prior knowledge, such as a sensor's saturation level as a hard cap~\cite{simon2024bayesian}.
\end{remark}

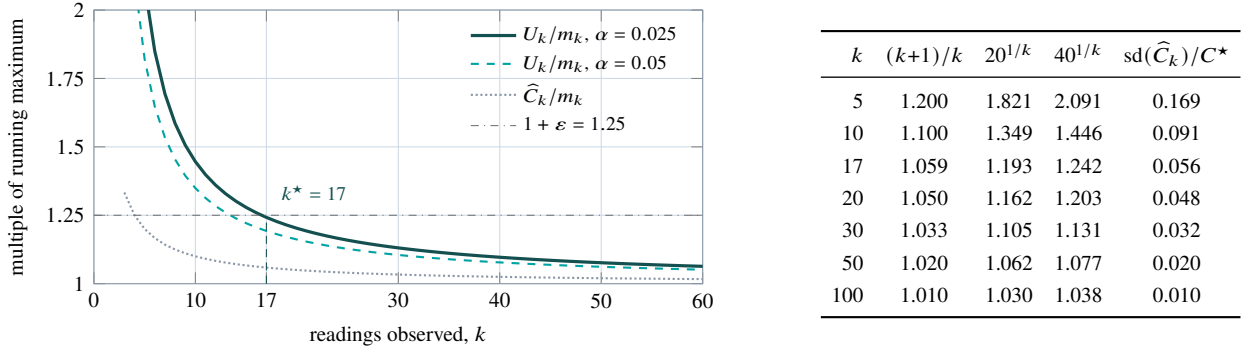
\begin{figure}[H]
\centering
\begin{minipage}[c]{0.56\linewidth}
\centering
\begin{tikzpicture}
\begin{axis}[
    width=\linewidth, height=5.2cm,
    xmin=0, xmax=60, ymin=1, ymax=2,
    xlabel={readings observed, $k$},
    ylabel={multiple of running maximum},
    xtick={0,10,17,30,40,50,60},
    ytick={1,1.25,1.5,1.75,2},
    yticklabel style={/pgf/number format/fixed,/pgf/number format/precision=2},
    tick label style={font=\footnotesize},
    label style={font=\footnotesize},
    grid=major, major grid style={draw=scrule},
    axis line style={scgray},
    legend style={font=\scriptsize, draw=none, fill=none, at={(0.98,0.98)}, anchor=north east},
    legend cell align=left,
  ]
  \addplot[scnavy, very thick, domain=3:60, samples=58] {exp(ln(40)/x)};
  \addlegendentry{$U_k/m_k$, $\alpha=0.025$}
  \addplot[sclink, thick, dashed, domain=3:60, samples=58] {exp(ln(20)/x)};
  \addlegendentry{$U_k/m_k$, $\alpha=0.05$}
  \addplot[scgray, thick, densely dotted, domain=3:60, samples=58] {(x+1)/x};
  \addlegendentry{$\widehat{C}_k/m_k$}
  \addplot[black!55, thin, dashdotted, domain=0:60, samples=2] {1.25};
  \addlegendentry{$1+\varepsilon=1.25$}
  \draw[scnavy, densely dashed, thin] (axis cs:17,1) -- (axis cs:17,1.2423);
  \node[font=\scriptsize, text=scnavy, anchor=south west] at (axis cs:17.5,1.27) {$k^{\star}=17$};
\end{axis}
\end{tikzpicture}
\end{minipage}\hfill
\begin{minipage}[c]{0.41\linewidth}
\centering\footnotesize
\setlength{\tabcolsep}{4pt}
\begin{tabular}{rcccc}
\toprule
$k$ & $(k{+}1)/k$ & $20^{1/k}$ & $40^{1/k}$ & $\mathrm{sd}(\Chat)/\Cs$\\
\midrule
5   & 1.200 & 1.821 & 2.091 & 0.169\\
10  & 1.100 & 1.349 & 1.446 & 0.091\\
17  & 1.059 & 1.193 & 1.242 & 0.056\\
20  & 1.050 & 1.162 & 1.203 & 0.048\\
30  & 1.033 & 1.105 & 1.131 & 0.032\\
50  & 1.020 & 1.062 & 1.077 & 0.020\\
100 & 1.010 & 1.030 & 1.038 & 0.010\\
\bottomrule
\end{tabular}
\end{minipage}
\caption{How fast the evidence accumulates.
\textbf{Left:} the bound as a multiple of the running maximum, $\Uk/\mk=\alpha^{-1/k}$, and the point-estimate factor $(k+1)/k$. At $\alpha=0.025$, as in~\cite{france2026chasing}, the bound falls below $1+\eps=1.25$ at $\kstar=17$.
\textbf{Right:} the same factors as a firmware lookup table ($20^{1/k}$ and $40^{1/k}$ are $\alpha^{-1/k}$ at $\alpha=0.05$ and $0.025$), with the relative standard deviation of $\Chat$.}
\label{fig:bound}
\end{figure}

\section{Examples from real experiments}\label{supp:examples}
\begin{table}[H]
\centering\small
\caption{Navigation times for a UAV navigating a course to localize the source of a chemical compound using olfaction only and olfaction-vision hybrid sensing, as reported in~\cite{france2026chasing}.
Each trial used the unbiased estimator presented in this paper, and was critical to achieving fast and reliable navigation times.
Every trial ended with the UAV landing within 1\,m of the source. OIO: olfactory inertial odometry; E.\,SARSA($\lambda$): Expected SARSA($\lambda$).}
\label{tab:results}
\begin{tabular}{lllrrr}
\hdrow \thd{Modality} & \thd{Sensor} & \thd{Policy} & \thd{Mean (s)} & \thd{Std.\ dev.\ (s)} & \thd{Best (s)}\\
Olfaction & MOX & OIO & 98.38 & 14.84 & 71.59\\
\rowcolor{scrowalt} Olfaction & EC & OIO & 112.22 & 16.70 & 97.65\\
Olfaction & MOX & E.\,SARSA($\lambda$) & 103.99 & 14.89 & 82.73\\
\rowcolor{scrowalt} Olfaction & EC & E.\,SARSA($\lambda$) & 121.85 & 16.77 & 97.10\\
Olfaction + Vision & MOX & OIO & 94.03 & 14.68 & 66.12\\
\rowcolor{scrowalt} Olfaction + Vision & EC & OIO & 107.71 & 17.07 & 80.50\\
\end{tabular}
\end{table}